\documentclass{llncs}

\usepackage{graphicx}
\usepackage{amsmath, amssymb, amsfonts}
\usepackage{xfrac}
\usepackage[sectionbib,numbers]{natbib}
\usepackage{algorithm}
\usepackage[noend]{algorithmic}
\usepackage[titlenumbered,ruled,algo2e]{algorithm2e}
\usepackage{subcaption}
\usepackage{url}
\usepackage{ulem}
\usepackage{chngcntr}
\counterwithin{figure}{section}
\usepackage{float}
\usepackage{multirow}
\restylefloat{table}
\usepackage{color}\definecolor{darkblue}{rgb}{0,0,0.75}
\usepackage[pdftex,pdftitle="Introducing Noise in Decentralized Training of Neural Networks",colorlinks=true, bookmarksnumbered, citecolor=darkblue, urlcolor=darkblue, linkcolor=darkblue, pdfpagemode=UseNone]{hyperref}
\usepackage{wrapfig}
\usepackage{caption} 
\usepackage{setspace}
\usepackage{nicefrac}

\newcommand{\round}{t}

\newcommand{\learner}{i}
\newcommand{\totalLearners}{m}

\newcommand{\uprule}{\varphi}

\newcommand{\model}{f}

\newcommand{\sample}{x}
\newcommand{\truelabel}{y}

\newcommand{\initnoiselevel}{\epsilon_1}
\newcommand{\noiselevel}{\epsilon_\round}
\newcommand{\noise}{\psi}

\newcommand{\repeatthanks}{\textsuperscript{\thefootnote}}

\title{Introducing Noise in Decentralized Training of Neural Networks}
\date{}
\author{
Linara Adilova\inst{1,2}\thanks{These authors contributed equally.} \and
Nathalie Paul\inst{1}\repeatthanks \and
Peter Schlicht\inst{3}
}
\institute{Fraunhofer IAIS \email{<name>.<surname>@iais.fraunhofer.de}
\and
Fraunhofer Center for Machine Learning
\and
Volkswagen Group Research \email{<name>.<surname>@volkswagen.de}
}

\begin{document}

\maketitle

\begin{abstract}
It has been shown that injecting noise into the neural network weights during the training process leads to a better generalization of the resulting model. Noise injection in the distributed setup is a straightforward technique and it represents a promising approach to improve the locally trained models. We investigate the effects of noise injection into the neural networks during a decentralized training process. We show both theoretically and empirically that noise injection has no positive effect in expectation on linear models, though. However for non-linear neural networks we empirically show that noise injection substantially improves model quality helping to reach a generalization ability of a local model close to the serial baseline.

\end{abstract}

\section{Introduction}
\label{sec:intro}

The idea of noisy optimization originates from the physical process of annealing \citep{kirkpatrick1983optimization}, where noise helps to stabilize the state of the system through changing it by small random perturbations. Injecting noise in various ways is exploited both in convex problems \citep{ben1998robust} and non-convex ones, specifically, in the training of neural networks \citep{an1996effects}. Research empirically shows that noise injection into the neural network objective optimization benefits the training process and improves generalization \citep{edwards1998fault, neelakantan2015adding}. In our work we consider the particular case of noise injection into the model parameters, i.e., the neural network weights.

Optimization of the objective function for various tasks can also be performed in a decentralized manner, i.e., multiple learners train on distributed data sources and synchronize according to a chosen schedule \citep{kamp2014communication, mcmahan2016communication, kamp2016communication}. Decentralized training has a strong motivation coming from a growing amount of devices, e.g., mobile phones, with possibly privacy sensitive data sources and the ability to perform local computations. However, models obtained via synchronizing results of local training often do not reach the performance of serial baseline trained on a centralized dataset. 

In order to tune the quality of the distributed training setup various synchronizing protocols were investigated, which address communication timing and aggregating operators \citep{kamp2017effective, kamp2014communication, kamp2016communication}. We consider using an existing approach for improving serial training, namely noise injection, for the case of decentralized setup in order to reach higher accuracies of the local models.

We give an overview of the related research in Section~\ref{sec:background}. For prior theoretical investigation we prove in Section~\ref{sec:theory} that in the special case of linear neural networks zero-mean noise does not have an adversarial effect on the results of training, because it cancels out in expectation. This is supported by experiments in Section~\ref{ssec:lin_exps}. Since non-linear neural networks have more practical usage, in Section~\ref{ssec:nlin_exps} we empirically study noise injection to non-linear neural networks trained in the decentralized manner and show that it improves performance. The simplest case of noise injection to the neural network weights is initialization noise. Experiments of \citet{mcmahan2016communication} suggest that averaging of independently trained neural networks on different local datasets results in a model that performs worse than any model of the local learners. However, when averaging periodically, our experiments show that noisy initialization is indeed beneficial. Further experiments in Section~\ref{ssec:nlin_exps} are concerned with noise injection during each synchronization step and show an improved quality of the trained models compared to the non-noisy setup for two considered classification tasks. Section~\ref{sec:conclusion} summarizes the research results and suggests possible future work.

\section{Related Work}
\label{sec:background}


Various ways of noise injection into the process of neural networks training were thoroughly investigated in multiple research works. For example, \citet{bishop1995training} has shown that additive noise on the inputs is equivalent to a regularization term in a loss function if the noise amplitude is sufficiently small. \citet{an1996effects} explored noise injection to the inputs, outputs, and weights of neural networks. He observed that additive noise injection to the weights on each update step, either in on-line or in batch gradient descent optimization, leads to higher generalization of the learned solution. 
It was shown that noise added to the inputs affects only the smoothness of the resulting function, while noise injected to the weights also punishes large values and activations. Injected output noise changes the loss function only by a constant value and thus does not affect the quality of the learned function. Later \citet{wang1999training} demonstrated that noise added to the desired signal affects the variance of weights and leads to faster convergence. A different approach to noise injection to the training process is disturbing the gradient updates. It is investigated, for example, in works of \citet{an1996effects} or \citet{neelakantan2015adding}. Empirical results show that adding decaying noise to gradients helps achieving a global minimum, but does not affect the generalization quality of the resulting network. Yet reaching a global minimum might be even harmful for generalization, since such solutions tend to overfit on the training data.

Decentralized training of neural networks on distributed data sources with different ways of synchronization is investigated in works of \citet{mcmahan2016communication}, \citet{jiang2017collaborative}, \citet{smith2017federated} and some others. The issue closest to our research is the initialization aspect considered by \citet{mcmahan2016communication}. The process of the neural network objective optimization is sensitive to the initial state of the model as it was presented by \citet{kolen1991back}. According to \citet{mcmahan2016communication} the averaging is harmful if initialization of the locally trained models is different, which corresponds to injection of random noise to initially equal states. Nevertheless the effect of initialization noise together with periodic averaging was not investigated. 


\section{Noisy Averaging}
\label{sec:theory}
To start our theoretical investigation we consider the simplest case of models, that is linear models, and analyze the effect of noise injection into their parameters for a common training setup. Here we present the summarization of used basic concepts and a proof that noise injection in the considered special case for linear models has no detrimental effect on the final model.

\subsection{Basic Concepts}\label{ssec:basic_concepts}

The task of training a neural network is an optimization problem which is commonly solved by using \textit{stochastic gradient descent} (SGD) or SGD-based algorithms.
SGD is an on-line optimization algorithm in which a local minimum of the objective function is determined by moving in the negative direction of its gradient \citep{chung1954stochastic}. SGD is derived from the gradient descent algorithm and estimates the gradient by considering only one training example.
Let $X$ denote the \textit{input space}, $Y$ the \textit{output space}, $\mathcal{F}$ the \textit{model parameters space} and $\model \in \mathcal{F}$ the model parameters. The SGD update rule for minimizing an \textit{objective function} $\ell: \mathcal{F} \times X \times Y \rightarrow \mathbb{R}_+$ reads

\begin{equation*}
	\uprule_{\eta}^{SGD}(\model) = \model - \eta \nabla \ell(\model, \sample, \truelabel),
\end{equation*}
where the learning rate $\eta$ controls the step size.

In our work we consider \textit{mini-batch SGD} which approximates the gradient by taking batch size $B$ many training examples into account:

\begin{equation}\label{batchSGD}
	\uprule_{B,\eta}^{mSGD}(\model) = \model - \eta \sum_{j=1}^{B}\nabla \ell^j(\model),
\end{equation}

where $\ell^j(\model) = \ell(\model, \sample_j, \truelabel_j).$

An example of a loss function is \textit{squared loss} 
$
	\ell^j(\model) = \nicefrac{1}{2}\Big(\langle \model, x_j \rangle -y_j\Big)^2
$.

Choosing this loss function and computing its gradient, we can write the learning algorithm update rule in the following form:

\begin{equation*}
	\uprule_{B,\eta}^{mSGD}(\model) = \model - \eta \sum_{j=1}^{B}\bigg(\Big(\langle \model, x_j \rangle-y_j \Big)x_j \bigg).
\end{equation*}

\subsection{Periodic Averaging Protocol with Noise Injection}\label{ssec:theory_alg}

As a protocol for the theoretical investigation of the effect of noise injection in the decentralized setup we consider the periodic averaging protocol with zero-mean noise injection (Algorithm~\ref{alg:protocol}). Let $\Psi$ denote a probability distribution with mean zero and $\noiselevel \geq 0$ the time-dependent noise level factor.

\begin{algorithm2e}[ht]
	\caption{Periodic averaging protocol with zero-mean noise injection}
	\label{alg:protocol}
    \smallskip
    \textbf{Input:}    batch size $b$\\    
	\smallskip
	\textbf{Initialization:}\\
	\begin{algorithmic}[0]
		\STATE local models $\model^1_1,\dots,\model^\totalLearners_1 \leftarrow$ one random $\model$
	\end{algorithmic}
	\smallskip
    
	\textbf{At round }$\round$\textbf{:}\\
    
	\begin{algorithmic}[0]
    	\STATE \textbf{At learner }$\learner$\textbf{:}\\
        \Indp
        \vspace{0.05cm}
        \STATE \textbf{observe} $\left(\sample^\learner_{\round,1},y^\learner_{\round,1}\right),\dots,\left(\sample^\learner_{\round,B},\truelabel^\learner_{\round,B}\right)$\\
        \vspace{0.02cm}
		\STATE $\noise_{\learner} \leftarrow \Psi$
        \STATE \textbf{update} $\model^\learner_{\round}$ using the learning algorithm $\uprule: \model^{\learner}_{\round +1} \leftarrow \uprule\left(\model^{\learner}_{\round} + \noiselevel\noise_{\learner}\right)$\\
        \Indm
        \vspace{0.05cm}
        \STATE \textbf{If} $\round\mod b=0$\textbf{:}\\
        \Indp
        \STATE \textbf{synchronize local models: } $\model^{\learner}_{\round +1} \leftarrow \frac{1}{m}\sum_{i=1}^{m} \model^{\learner}_{\round +1}$
	\end{algorithmic}
\end{algorithm2e}

In the following we analyze the influence of noise injection on the behavior of the learning algorithm for the special case of linear models. For this case we can prove that adding zero-mean noise to the parameters optimized by mini-batch SGD with squared loss does not change the model parameters in expectation.

\noindent

\begin{lemma}
	For a linear model let $\model_\round$ denote the model parameters attained by using mini-batch SGD with squared loss and scaled additive zero-mean weight noise injection, i.e.
	\begin{align*}
    \model_{\round +1} = \uprule_{B,\eta}^{mSGD}(\model_{\round} + \noiselevel\noise).
	\end{align*}
    Let $g_\round$ denote the model parameters attained by using common mini-batch SGD, i.e.
    \begin{align*}
    	g_{\round +1} = \uprule_{B,\eta}^{mSGD}(g_{\round}).
	\end{align*}
    If the learning algorithms are identically initialized, it holds
	\begin{equation}
		\mathbb{E}\big[\model_\round \big] = g_\round \quad\quad \text{for all } t=1,...,T.
	\end{equation}
\end{lemma}

\begin{proof}
	We use induction over $t$.
	The case $t=1$ follows immediately since by initialization the model parameters trained with and without noise are the same.
	Applying the definitions above, the expected model parameters for mini-batch SGD with noise injection read
        
        \begin{align*}
        	\mathbb{E}\Big[\model_\round \Big] & = \mathbb{E}\Big[\uprule_{B,\eta}^{mSGD}\left(\model_{\round-1} + \noiselevel\noise\right) \Big] \\
            &= \mathbb{E}\bigg[\model_{\round-1} + \noiselevel\noise - \eta \sum_{j=1}^{B}\bigg(\Big(\langle \model_{\round-1}+ \noiselevel\noise, \sample_{(\round-1,j)} \rangle-\truelabel_{(\round-1,j)} \Big)\sample_{(\round-1,j)} \bigg) \bigg] \\
            &= \noiselevel\underbrace{\mathbb{E}\big[\noise \big]}_{=0} + \mathbb{E}\bigg[\model_{\round-1} - \eta \sum_{j=1}^{B}\bigg(\Big(\langle \model_{\round-1}, \sample_{(\round-1,j)} \rangle-\truelabel_{(\round-1,j)} \Big)\sample_{(\round-1,j)} \bigg) \bigg] \\ 
            &\qquad - \underbrace{\mathbb{E}\bigg[ \eta \sum_{j=1}^{B}\Big(\noiselevel \langle\noise, \sample_{(\round-1,j)} \rangle \sample_{(\round-1,j)} \Big) \bigg]}_{=0},
        \end{align*}
        
\noindent
where we used that $\noise$ is centered.
By employing the induction assumption we get $\mathbb{E}\big[\model_{\round-1}\big] = g_{\round -1}$ and conclude
        
    \begin{align*}
    	\mathbb{E}\Big[\model_\round \Big] & = g_{\round-1} - \eta \sum_{j=1}^{B}\bigg(\Big(\langle g_{\round-1}, \sample_{(\round-1,j)} \rangle-\truelabel_{(\round-1,j)} \Big)\sample_{(\round-1,j)} \bigg) \\
        &= \uprule_{B,\eta}^{mSGD}(g_{\round-1}) \\
        &= g_{\round}.
	\end{align*}    
	\qed 
\end{proof}      

The equivalence of Algorithm~\ref{alg:protocol} to the non-noisy periodic averaging protocol in expectation directly follows from the linearity of the expected value.
                                                                                   
\begin{corollary}
	For linear models trained with mini-batch SGD and squared loss given identical initialization, the expected model obtained by the periodic averaging protocol with zero-mean noise injection is equivalent to the model obtained by the non-noisy periodic averaging protocol.
\end{corollary}

For the specific considered case of linear models we observe that noise injection has no influence on the model parameters in expectation both for serial and distributed learning.
This theoretical result is supported by empirical evidence in the following Section.
In contrast, for non-linear models the research of \citet{an1996effects,edwards1998fault} has shown that noise injection in serial training helps improving generalization. We conjecture that in the distributed training setup injecting noise into non-linear models might also improve generalization properties of the obtained solution compared to the non-noisy training.
We leave a theoretical investigation of the effect of noise injection into non-linear models in distributed training setup for future work.
To substantiate our conjecture, in the following Section we perform an empirical analysis of zero-mean noise injection in a distributed setup for non-linear neural networks.


\section{Empirical Evaluation}
\label{sec:exps}
To investigate the effect of noise injection into the neural networks trained in a decentralized manner we performed a set of experiments described further.

The decentralized setup of these experiments is periodic synchronization via averaging (cf. Algorithm~\ref{alg:protocol}). Apart from the distributed synchronized models two baselines are trained: a local model without any synchronization with other local learners (no-sync) and a model with full data centralization (serial). These baselines are necessary to assess the performance of the synchronizing local learners, since synchronization aims to reach the performance of the serial model and outperform the no-sync baseline. Noise injection into the serial baseline is also a subject of interest, allowing to compare the possible gains in generalization ability in centralized and distributed case. For evaluation of the synchronizing learners the last averaged model obtained during the training process is used. For evaluation of the no-sync baseline we pick one random model among locally trained ones. 

\begin{wrapfigure}{r}{0.4\textwidth}
	\includegraphics[width=0.4\textwidth]{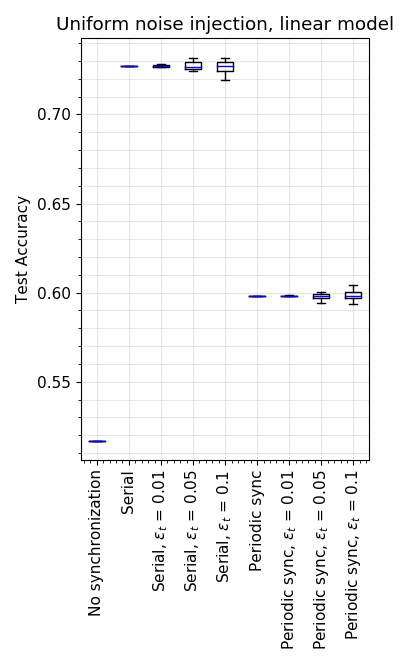}
	\captionof{figure}{Effect on test accuracy of noise injection to local learners and to the serial baseline.}
	\label{fig:linear-noise}
\vspace{-1cm}
\end{wrapfigure}

Since we explore the effect of noise injection, we are interested in the behavior of models trained throughout several experiments that differ only by the used noise. Thus all the setups were run $10$ times without using fixed random seed. This produces an indication of how distribution of possible outcomes of the training process looks like. The results are presented in the form of box plots. Here the box shows the observed values from the first to the third quartile, with a line at the median. The whiskers show the range of the results and points are representing outliers.

\subsection{Linear Neural Networks}
\label{ssec:lin_exps}

First set of experiments is performed to empirically evaluate the theoretical result obtained in Section~\ref{sec:theory}. We employed a linear neural network with three layers having $32$, $64$ and $1$ neuron correspondingly for approximating the target column of SUSY dataset \citep{baldi2014searching}. For linear model experiments the dataset was normalized thus having $-1$ and $1$ as targets and accuracy was calculated with $0$ threshold. The optimal training parameters determined were mini-batch of $B = 10$ examples, learning rate of $\eta=1\mathrm{e}{-5}$ for the serial and no-sync baselines and $\eta=2.5\mathrm{e}{-5}$ for the local learners. We employed squared loss as explained in Section~\ref{sec:theory}. During the training each local learner was presented $20000$ examples, while serial baseline had $m \cdot 20000$ examples with $m = 10$. The noise used for this experiment is additive uniform noise in the range $[-0.5, 0.5]$ with decay factor equal to the synchronization round number. Absence of decay factor was leading to fast overflow thus making experiments not runnable.

Figure~\ref{fig:linear-noise} shows the test accuracy evaluated for $10$ runs for the baselines and synchronizing learners on the independent test set of $1000000$ examples. We can observe that larger noise injected into the weights of the models leads to larger variance of obtained accuracy at the same moment leaving the median value throughout different setups the same. This empirically supports the effect of noise cancellation in expectation for linear models in this training setup.

\subsection{Non-linear Neural Networks}
\label{ssec:nlin_exps}

In the following we evaluate noise injection to the decentralized training of non-linear neural networks on the basis of two classification tasks. For our experiments we choose to add uniform noise in range $[-0.5, 0.5]$ and Gaussian noise in the same range. 

\subsubsection{Binary Classification}
For preliminary evaluation of the approach for non-linear case we have considered the classification task on the SUSY dataset. In contrast to the linear case, the employed model architecture is a three-layered dense network with sigmoid activations. The first layer has $32$ neurons, the second $64$ and the output layer has $2$ neurons with softmax activation. We have determined the optimal parameters of the non-noisy learning algorithm on a small fraction of the dataset. That is training mini-batch of $B = 10$ examples, learning rate of $\eta=0.1$ for the serial and no-sync baselines and $\eta=0.25$ for the local learners.

\paragraph{Initialization Noise}
The simplest way of noise injection to the distributed neural networks is one step noise injection right after initialization.
In \citet{mcmahan2016communication} it was shown that such noise together with one-time synchronization after local training results in a worse model than each local one in terms of training loss. We want to investigate whether periodic synchronization is capable of being more robust to initial noise injection.

With regard to Algorithm~\ref{alg:protocol}, initialization noise corresponds to choosing $\initnoiselevel > 0$ and $\noiselevel = 0$ for all $t > 1$. It means that we add randomly sampled noise to each initial weight before starting the training process.

\begin{figure}
	\centering
	\begin{subfigure}{0.48\textwidth}
		\centering
		\includegraphics[width=1\linewidth]
{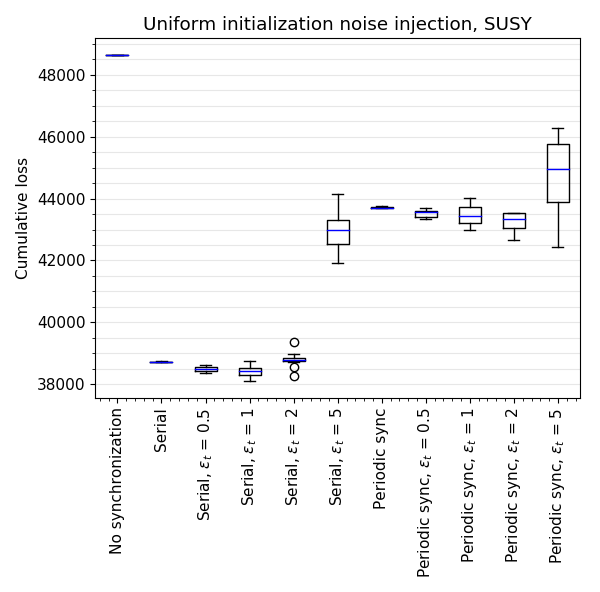}
		\label{fig:init_sub1}
	\end{subfigure}
    \quad
	\begin{subfigure}{0.48\textwidth}
		\centering
		\includegraphics[width=1\linewidth]{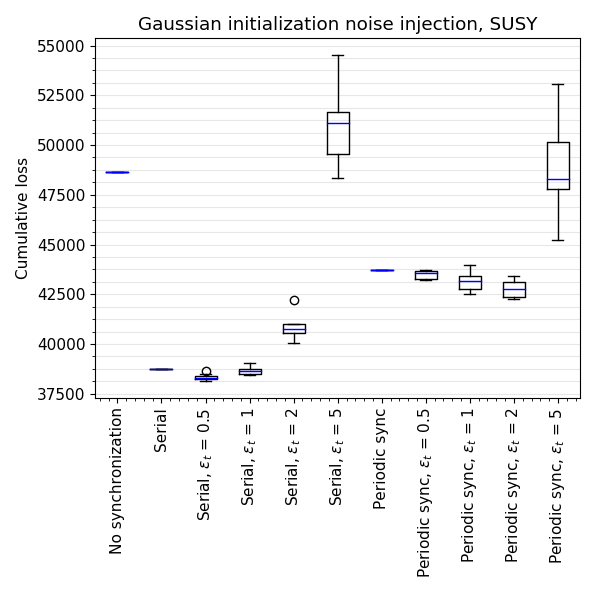}
		\label{fig:init_sub2}
	\end{subfigure}
	\caption{Effect on cumulative training loss of uniformly and normally distributed noise added to the equal initializations of local learners and to the serial baseline.}
	\label{fig:init-noise}
\end{figure}

Figure~\ref{fig:init-noise} shows that both Gaussian and uniform initialization noise improves the serial and periodically synchronizing models in terms of cumulative training loss when using it up to some small extent ($\noiselevel < 1$). On the contrary, higher levels of noise (e.g. $\noiselevel = 5$) make training harder in each setup. Interestingly, a noise level of $\noiselevel$ = 2 for both Gaussian and uniform noise leads to a higher cumulative training loss in the serial model while the distributed setup benefits from it.

Even though the Gaussian distribution is a very popular choice for initializing neural networks (\citet{glorot2010understanding}), adding large levels noise drawn from the normal distribution deteriorates the training process worse than uniformly distributed noise. One possible explanation is that the initial weights are already distributed normally according to the best practices and additional Gaussian noise intervenes with it in a destructive way. 


Initialization noise experiments reveal that one-time noise injection to the initially equal model weights helps the training process since the cumulative training loss is decreasing. This motivates follow-up experiments with continuous noise injection.


\paragraph{Continuous Noise}
Extending the initialization noise setup we now perform additional noise injection steps: Zero-mean noise gets injected to the local models' weights after every synchronization step.  
Formally, in Algorithm~\ref{alg:protocol} it corresponds to $\initnoiselevel > 0$, $\noiselevel > 0$ for all $t \mod b = 0$ and $\noiselevel = 0$ otherwise. Following the work of \citet{murray1994enhanced} the noise is decaying and the decaying factor is equal to the index number of the synchronization step, i.e. noise level is given by $\nicefrac{\noiselevel}{t}$. 

We want to explore whether continuous noise injection improves the generalization ability of the resulting models in the distributed setup. Therefore we calculated the evaluation accuracy on an independent test set of $1000000$ examples for each of the trained models. During training, each of the local learners $i$ is presented $1000$ examples from the training dataset, while the serial model sees $m\cdot 1000$ examples. Various setups together with evaluated validation accuracies are depicted as box plots in Figure~\ref{fig:contin-noise-susy}.

\begin{figure}[H]
	\centering
	\begin{subfigure}{0.48\textwidth}
		\centering
		\includegraphics[width=1\linewidth]
{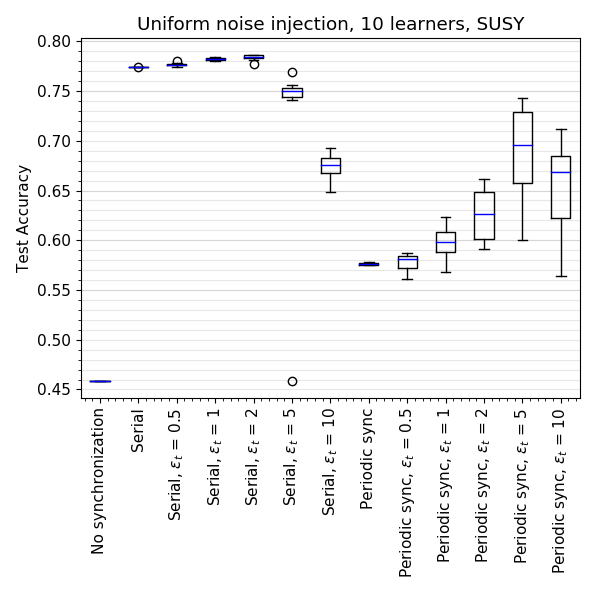}
		\label{fig:contin-noise-susy_sub1}
	\end{subfigure}
    \quad
	\begin{subfigure}{0.48\textwidth}
		\centering
		\includegraphics[width=1\linewidth]{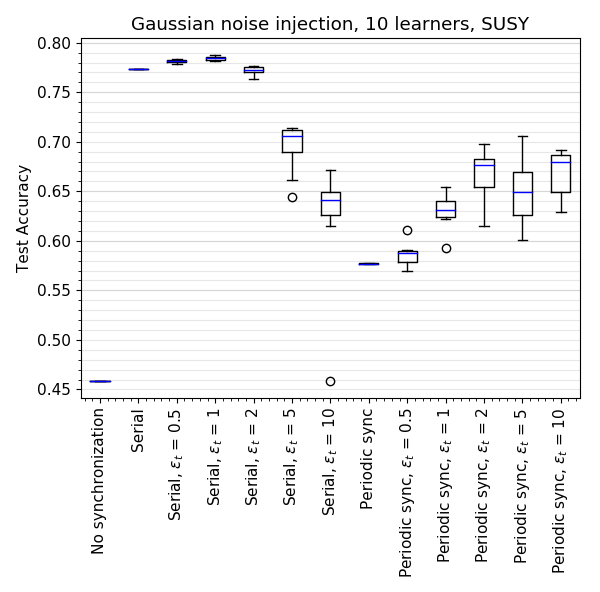}
		\label{fig:contin-noise-susy_sub2}
	\end{subfigure}
	\begin{subfigure}{0.48\textwidth}
		\centering
		\includegraphics[width=1\linewidth]
{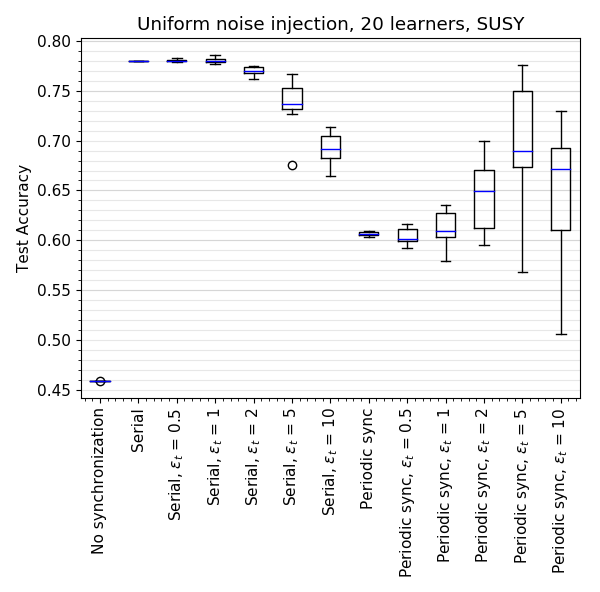}
		\label{fig:contin-noise-susy_sub3}
	\end{subfigure}
    \quad
	\begin{subfigure}{0.48\textwidth}
		\centering
		\includegraphics[width=1\linewidth]{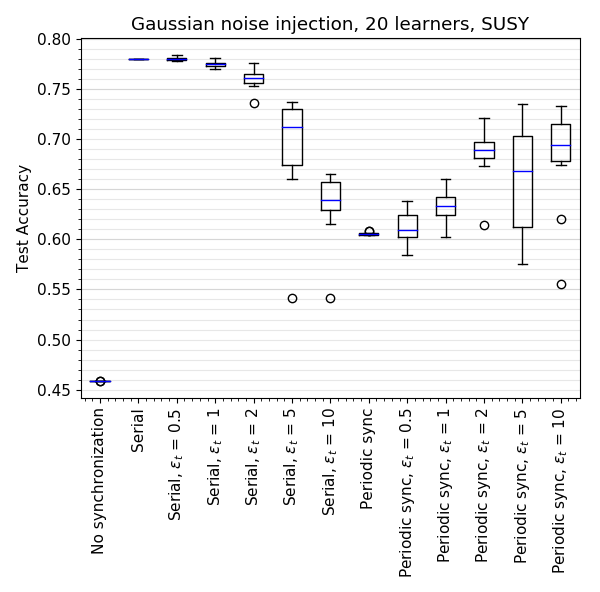}
		\label{fig:contin-noise-susy_sub4}
	\end{subfigure}
	\caption{Effect on test accuracy of injecting uniformly or normally distributed noise every $b$ time steps throughout the training process of local learners and the serial baseline.}
	\label{fig:contin-noise-susy}
\end{figure}

The evaluation shows that noise injection can substantially improve the generalization ability of the models. When comparing results of the setup with $10$ and $20$ learners, we see that a larger amount of learners leads to a larger spread of the results of the training process that can mean either a better generalization or to the contrary convergence to a worse model.

One can also observe that with growth of the uniform noise level the resulting test accuracy becomes more unstable, i.e., for having a possibly higher median we get a larger range of values.
In the experiments, Gaussian noise is more stable than uniform noise while on the other hand it has more outliers below the median. The spread of the serial baseline with noise injection is very small compared to the distributed models. It might be interesting to investigate the reasons why noise has a more pronounced effect in the distributed setup than in the serial one.


\subsubsection{Image Classification}

\begin{figure}[H]
	\centering
	\begin{subfigure}{0.48\textwidth}
		\centering
		\includegraphics[width=1\linewidth]
{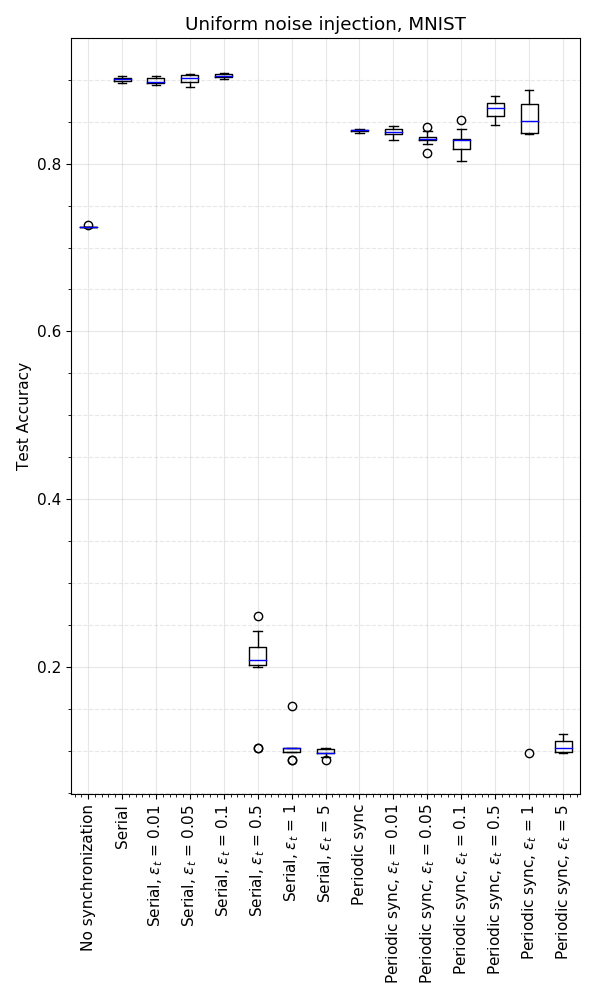}
		\label{fig:contin-noise-mnist_sub1}
	\end{subfigure}
    \quad
	\begin{subfigure}{0.48\textwidth}
		\centering
		\includegraphics[width=1\linewidth]{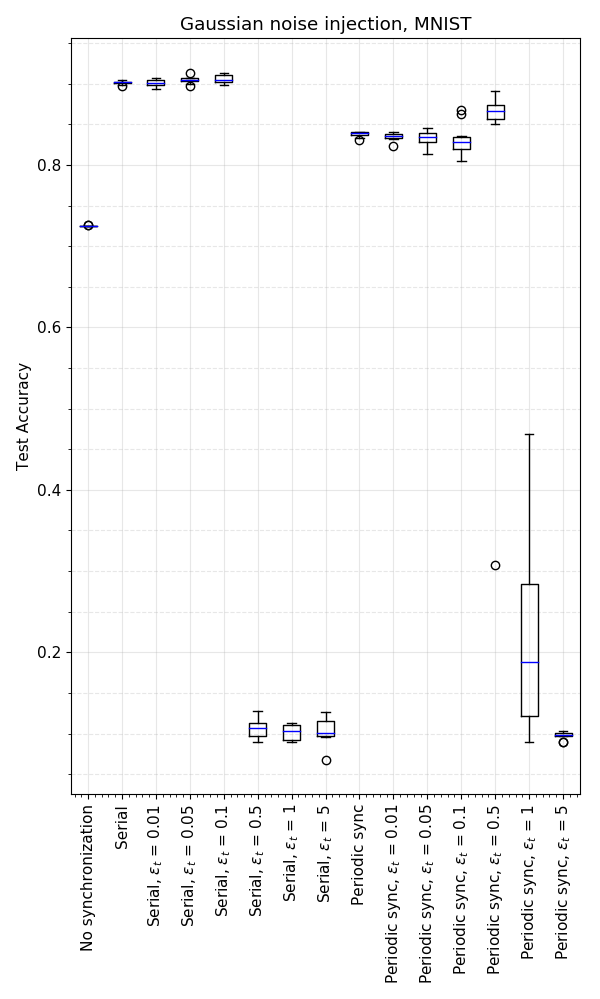}
		\label{fig:contin-noise-mnist_sub2}
	\end{subfigure}
	\caption{Effect on test accuracy of injecting uniformly and normally distributed noise every $b$ time steps throughout the training process of local learners and the serial baseline.}
	\label{fig:contin-noise-mnist}
\end{figure}

To further investigate the effect of noise injection into non-linear models in a distributed setup we have chosen the classification task on the MNIST dataset \citep{lecun1998mnist}. The model architecture is more complicated than in the previous experiment. It has two dense layers with $512$ neurons each and a dropout layer after each of them. The output layer performs a softmax activation to predict one of the ten classes. The activation of the dense layers is ReLU.
We have determined the optimal parameters of the non-noisy learning algorithm on a small fraction of the dataset. That is, equivalently to the first experiment, training mini-batch of $B = 10$ examples, learning rate of $\eta=0.1$ for the serial and no-sync baselines and $\eta=0.25$ for the local learners.
During training each learner is presented $500$ examples from the training set and evaluation is performed on the $10000$ images of the test set.

We observe that in this experiment, compared to the previous, noise injection requires different level values in order to obtain the improved quality of the models. More precisely, for noise levels $\noiselevel \ge 1$ we get prohibitively low test accuracies for both uniform and Gaussian noise. One possible explanation is that the ReLU activation is much more sensitive to noise disturbance of the weights compared to the bounded sigmoid function which we have used in the first experiment. Moreover the dropout layers might be contributing to this effect since dropout is also supposed to prevent a model from overfitting \citep{srivastava2014dropout}. More refined research on this effect is left for future work.
Concentrating on lower levels of noise ($\noiselevel < 0.5$), we see that we can again improve the generalization ability of the trained models. In the distributed setup this effect is more pronounced than in the serial one similarly to the first set of experiments.

\section{Conclusion}
\label{sec:conclusion}

The research presented in this paper investigates noise injection into neural networks, trained in a decentralized manner on distributed data sources. We have proven that for linear models in a common training setup zero-mean noise injection retains the results of the non-noisy setup.
We have performed experiments to empirically underline this theoretical statement.
Further experiments show that with non-linear models in a distributed setup noise injection improves the quality of the models. The evaluation shows that indeed carefully chosen levels of noise have a positive effect on the generalizing abilities of the synchronized models. It might be explained by the fact that noise enforces wider exploration of the space of solutions \citep{gan2017scalable, chen2017noisy} which is an interesting subject for further investigation. 
Also, experiments show that the impact of noise in the distributed training is even greater than in the serial case.

Future research could investigate the theoretical background of noise injection to non-linear models in a decentralized training setup as well as the effect of various network architectures and training parameters. A promising framework for studying noise injection effects is regularization theory, since injected noise can be described as a regularization term that is added to the loss function \citep{an1996effects,bishop1995training}.


\newpage

\bibliographystyle{plainnat}
\bibliography{bibliography}

\end{document}